\documentclass[conference]{IEEEtran}
\usepackage{amsmath,amssymb}
\usepackage{nccmath}
\usepackage{amsmath,amssymb}
\usepackage{mathtools}
\usepackage{xcolor}
\usepackage{graphicx}
\usepackage{subcaption}
\usepackage{etoolbox}

\usepackage[utf8]{inputenc} 
\usepackage[T1]{fontenc}  
\usepackage{mathrsfs}
\usepackage{%
algorithm,%
algpseudocode,%
amsfonts,%
amsmath,%
amssymb,%
amsthm,%
bbm,%
caption,%
cite,%
comment,%
extarrows,%
float,%
graphicx,%
hyperref,%
indentfirst,%
listings,%
paralist,%
pgfplots,%
subcaption,%
tikz,%
todonotes,%
url,%
xcolor,%
}
\usepackage{epstopdf}

\pgfplotsset{plot coordinates/math parser=false}
\usepgflibrary{%
patterns,%
}

\usetikzlibrary{%
arrows,%
automata,%
chains,%
decorations.pathreplacing,%
intersections,%
positioning,%
shapes,%
}

\IEEEoverridecommandlockouts

\newtheorem{theorem}{Theorem}
\newtheorem{corollary}{Corollary}

\newtheorem{assumption}{Assumption}

\theoremstyle{definition}
\newtheorem{definition}{Definition}

\theoremstyle{remark}
\newtheorem{remark}{Remark}

\definecolor{darkgreen}{rgb}{0.0, 0.5, 0.0}

\newcommand{\set}[1]{\left\lbrace#1\right\rbrace}

\newcommand{\E}{\mathbb{E}}

\newcommand{\R}{\mathbb{R}}

\newcommand{\cC}{\mathcal{C}}
\newcommand{\cD}{\mathcal{D}}
\newcommand{\cE}{\mathcal{E}}

\newcommand{\cG}{\mathcal{G}}
\newcommand{\cH}{\mathcal{H}}

\newcommand{\cM}{\mathcal{M}}

\newcommand{\cV}{\mathcal{V}}

\renewcommand{\ge}{\geqslant}
\renewcommand{\le}{\leqslant}

\newcommand{\iid}{\emph{i.i.d.}\ }

\title{Multi-Token Random-Walk Learning via Meeting-Triggered Information Exchange}

\author{
Taichun Yeh,\quad Xingran Chen 
\IEEEcompsocitemizethanks 	 
{
\IEEEcompsocthanksitem  Taichun Yeh is with Department of Electrical and Computer Engineering,  Rutgers University, Piscataway Township, NJ 08854, USA	(E-mail: taichun.yeh@rutgers.edu).
\IEEEcompsocthanksitem  Xingran Chen is with Engineering Systems and Design Pillar, Singapore University of Technology and Design, Singapore 487372, Singapore	(E-mail: xingranc@ieee.org).
}
}

\begin{document}

\maketitle

\begin{abstract}
Random walk (RW)-based algorithms are widely used in decentralized learning due to their low communication overhead and scalability. While most existing methods rely on a single token, multiple tokens can explore the network in parallel and exchange information upon meeting. However, the fundamental benefits of such meeting-triggered interactions remain unclear. In this work, we investigate multi-token RW learning under two information exchange policies: \textsc{Average-Gradient} and \textsc{Average-Model}, which average stochastic gradients and updated models, respectively, when tokens meet. We establish iteration complexity lower bounds for non-convex, convex, and strongly convex objectives on complete graphs, which apply to both exchange policies. Compared with single-token RW learning, the scaling factor in these bounds is reduced by a factor of $m^2$ for non-convex and convex objectives, and by an exponentially in $m$ for strongly convex objectives, where $m$ denotes the number of tokens. Numerical experiments reveal distinct behaviors of the two policies under two performance criteria: mean-token loss (the average loss across all tokens) and best-token loss (the minimum loss across all tokens). \textsc{Average-Model} achieves lower mean-token loss and is relatively insensitive to the network-to-token ratio $N/m$, where $N$ is the number of nodes. In contrast, \textsc{Average-Gradient} achieves lower best-token loss but exhibits a trade-off between parallel exploration and gradient interference as $N/m$ decreases.
\end{abstract}

\section{Introduction}

Decentralized learning enables a network of agents, each with access only to its local data and local network information, to collaboratively solve a global optimization problem without a central coordinator. Two widely studied approaches are consensus-based methods~\cite{koloskova2020unified} and random-walk (RW)-based methods~\cite{lovasz1993random,levin2017markov}. 

In consensus-based methods, each agent maintains a local model and updates it using its local data. Agents repeatedly exchange and aggregate model information with their neighbors to drive their local models toward a common solution. Although effective, such repeated communication can incur substantial aggregate communication overhead, particularly in large-scale networks~\cite{koloskova2020unified}. In contrast, RW-based methods employ a mobile token that carries a model along a random walk over the communication graph. At each iteration, the agent holding the token updates the model using its local data and forwards the token to a
randomly selected neighbor \cite{chen2026pacman, chen2026selfcreating}. Thus, only one agent performs a model update and one token transmission occurs per iteration, significantly reducing the instantaneous communication and computation costs. Random walks have been widely used in web ranking, recommendation, and link prediction ~\cite{page1999pagerank,fouss2007random,backstrom2011supervised}, and have also been adopted for decentralized machine learning~\cite{chen2026rwlora, johansson2007peer,ayache2023walk}.

Most existing works on RW-based decentralized optimization consider a single token~\cite{johansson2007peer,mao2020walkman,even2023stochastic}, which processes data sequentially along one trajectory. A natural extension is to run multiple tokens in parallel. A parallel random-walk-based alternating direction method of multipliers algorithm was proposed for decentralized consensus optimization~\cite{ye2020parallel}, while a multi-token coordinate descent algorithm with convergence guarantees was developed for semi-decentralized vertical federated learning~\cite{valdeira2026multitoken}. Beyond parallel execution, tokens that meet at the same node can exchange information collected along their trajectories. A conceptual-graph framework was proposed in~\cite{hendrikx2023token} to model multi-token optimization with meeting-based interactions.  However, by decoupling communication from computation, the framework considers a special setting that does not capture the inherent coupling between token movements and local model updates in RW learning. In contrast, we consider multi-token RW learning with meeting-triggered information exchange, where model updates are coupled with token movements and information is exchanged only when tokens meet. The fundamental performance limits and benefits of such interactions remain largely unexplored.

In this work, we investigate multi-token RW learning, where multiple tokens explore the network in parallel and exchange information upon meeting. We consider two information exchange policies, \textsc{Average-Gradient} and \textsc{Average-Model}, and analyze their theoretical limits and empirical performance. Our main contributions are summarized as follows:
\begin{enumerate}
\item \textbf{Fundamental iteration complexity limits.} We extend the Markov black-box framework of \cite{even2023stochastic} to establish iteration complexity lower bounds for multi-token RW learning on complete graphs under non-convex, convex, and strongly convex objectives. These bounds apply to both exchange policies and reveal how the fundamental limits scale with the network-to-token ratio $N/m$, where $N$ and $m$ denote the numbers of nodes and tokens, respectively. Compared with single-token RW learning, the scaling factor in the lower bounds is reduced by a factor of $m^2$ for non-convex and convex objectives, and exponentially in $m$ for strongly convex objectives.

\item \textbf{Policy-dependent effects of the network-to-token ratio.} We evaluate the two exchange policies using mean-token and best-token losses. \textsc{Average-Model} achieves better mean-token performance, whereas \textsc{Average-Gradient} achieves better best-token performance. As $N/m$ decreases, \textsc{Average-Gradient} benefits from parallel exploration early in training but exhibits deteriorating mean-token performance and a weakening best-token advantage later in training. In contrast, \textsc{Average-Model} remains relatively insensitive to $N/m$ over the tested range.
\end{enumerate}

\section{Problem Formulation}
\subsection{Network Model}
We consider a decentralized system consisting of $N$ agents that collaboratively perform a computational task without any central coordination. These agents are modeled as nodes in a connected undirected graph, where each agent possesses its own local data and can communicate only with neighboring agents. The communication topology graph is constructed as a finite graph $\cG \triangleq (\cV,\cE)$, where the set of nodes is $\cV \triangleq \{1,2,\ldots,N\}$, and $\mathcal{E} \subseteq \binom{\mathcal{V}}{2}$ denotes the set of edges.

In this system, computation is carried out via random walks (RWs) over the graph. Each RW evolves according to a fixed and identical transition probability matrix $P$. Moreover, each RW carries a token message, which can be viewed as representing a global computational task that is processed and passed through the network. At each time step, only the node currently holding a token performs local computation and updates the corresponding message. After completing the local computation, the current node forwards the token to one of its neighboring nodes according to the transition probabilities specified by \(P\). This process continues until a predefined stopping criterion is satisfied. When multiple RWs are present, each RW is uniquely indexed, e.g., by \(j\), in order to track its individual trajectory and progress through the network.

\subsection{RW-based Stochastic Gradient Descent}
Each node $u\in\cV$ has access to a local data distribution $\cD_u$. For a model parameter $x$ and a random data sample $\xi_u\sim\cD_u$, let $F_u(x;\xi_u)$ denote the stochastic loss at node $u$. The corresponding objective function is the expected stochastic loss $f_u(x)=\E_{\xi_u\sim\cD_u}\left[F_u(x;\xi_u)\right]$. The global optimization problem is 
\begin{align}\label{eq:distributed_optimization}
\min_{x \in \mathbb{R}^d}f(x)=\frac{1}{N}\sum_{u=1}^{N}f_u(x).
\end{align}
Each token $i$ is initialized at node $v_0^{(i)}\in\cV$ with model $x_0^{(i)}\in\R^d$. At iteration $t$, token $i$ is located at node $v_t^{(i)}$ and draws a random sample $\xi_u$ from $\cD_u$. Using this sammple, token $i$ computes a stochastic gradient or subgradient $g_t^{(i)}$ based on its current model $x_t^{(i)}$. The model is then updated according to
\begin{align}\label{eq:RW-SGD}
x_{t+1}^{(i)} = x_{t}^{(i)} - \eta_t g_t^{(i)},
\end{align}
where $\eta_t$ denotes the stepsize. For notational simplicity, we suppress the dependence of $g_t^{(i)}$ on the random sample $\xi_u$ whenever it is clear from context. After the update, token $i$ independently selects its next node from the neighbors, $v_{t+1}^{(i)}\sim P(v_t^{(i)},\cdot)$, and forward its updated model. Here, $P$ can be computed by the Metropolis-Hastings algorithm \cite{hastings1970monte}.

\subsection{Multi-token System and Token Meetings}
We are interested in investigating whether interactions among multiple random-walk tokens can improve convergence. In particular, when multiple tokens meet at the same node, they can exchange information by averaging their local models. This interaction allows information collected along different random-walk trajectories to be shared across tokens, potentially affecting the convergence behavior.

We focus on the time required to reach a prescribed accuracy level. The appropriate error criterion depends on the class of objective functions considered. For a target accuracy $\epsilon>0$ and token $i$, we denote $T_\epsilon^{(i)}$ as the first time at which the considered error metric falls below $\epsilon$.

For the non-convex case, convergence is measured in terms of first-order stationary \cite{carmon2021lower}
\begin{align}\label{eq:nonconvexTepsilon}
T_\epsilon^{(i)}\triangleq\inf\left\{t \ge 0 : \|\nabla f(x_t^{(i)})\|^2\le \epsilon\right\},
\end{align}
For the convex and strongly convex cases, we measure convergence by the objective gap \cite{nesterov2018lectures},
\begin{equation}\label{eq:convexTepsilon}
T_\epsilon^{(i)}\triangleq\inf\left\{t \ge 0 : f(x_t^{(i)})-f(x^*)\le \epsilon\right\},
\end{equation}
where $x^*$ denotes an optimal solution of $f$. Since all tokens are statistically identical, the random variables $T_\epsilon^{(i)}$ are identically distributed across $i$.

To study this effect of token interactions in a simple and analytically tractable setting, we consider two basic information-exchange policies. When multiple tokens meet at the same node, they can either aggregate the gradients or models (see Section~\ref{sec:InfoEx}). We compare both policies with a no-interaction baseline and characterize how such interactions affect the time $T_\epsilon^{(i)}$ required to reach a prescribed accuracy.

\subsection{Information-Exchange Policies}\label{sec:InfoEx}
Suppose that $K$ tokens are located at the same node $u \in \cV$ at time $t$. Upon meeting, the tokens exchange information through either gradient averaging or model averaging.

Since no additional information regarding the relative quality or reliability of the participating token is available, we treat all tokens equally and adopt uniform averaging. We consider the following two information-exchange policies.

\paragraph{Averaging Gradient} Each token first independently computes a stochastic gradient using its current model and a local sample from the visited node. The resulting gradients are then averaged, and all participating tokens use the same averaged gradient to perform their SGD updates. For $k\in[K]$,
\begin{align}\label{eq:AverageGradient}
\bar{g}_t=\frac{1}{K}\sum_{k=1}^{K} g_t^{(i_k)},\,\, x_{t+1}^{(i_k)}= x_t^{(i_k)}-\eta_t \bar{g_t}.
\end{align}

\paragraph{Model Averaging After SGD} Each token first independently performs an SGD update using its current model and a local sample from the visited node. The resulting updated models are then averaged. For $k\in[K]$:
\begin{align}\label{eq:AverageModel}
\widetilde{x}_{t+1}^{(i_k)}=x_t^{(i_k)}-\eta_tg_t^{(i_k)},\,\, x_{t+1}^{(i_k)}=\frac{1}{K}\sum_{k=1}^{K}\widetilde{x}_{t+1}^{(i_k)}.
\end{align}
Hence, the tokens may perform their local updates from different pre-meeting models, but they leave the meeting with the same averaged model.

As a baseline, we also consider the case without information exchange, in which each token independently performs RW-based SGD along its own trajectory.

\section{Convergence-Time Lower Bound on a Complete Graph}

Our analysis builds on the \textit{black box procedures}~\cite{even2023stochastic,nesterov2018lectures,carmon2021lower}. We extend this framework to the multi-token setting by allowing tokens to exchange information when they meet. Let a Markov chain $\set{v_t^{(i)}}_t$ on $\cV$ denote the location of token $i$ for each $i$.

\begin{definition}[Multi-token Markov black-box procedure]\label{def:MarkovBB}
Consider $m$ tokens following a Markov sampling scheme, meaning that they satisfy the following procedural constraints with information exchange for some fixed initialization $\cM_0=\set{x_0^{(i)}}_{i\in[m]}$. At iteration $t$:
\begin{compactenum}[(i)]
\item token $i$ has access to the local
function $f_{v_t^{(i)}}$ and extend its information set as
\begin{align*}
\widetilde{\mathcal M}_{t+1}^{(i)}=\span\left\{x,\nabla f_{v_t^{(i)}}(x):x\in\mathcal M_t^{(i)}\right\};
\end{align*}
\item let $\cC_t^{(i)}=\left\{j\in[m]:v_t^{(j)}=v_t^{(i)}\right\}$
denote the set of tokens meeting token $i$, and the information set of token $i$ can be further extended as
\begin{align*}
\cM_{t+1}^{(i)}=\span\left(\bigcup_{j\in\mathcal C_t^{(i)}}\widetilde{\mathcal M}_{t+1}^{(j)}\right);
\end{align*}
\item the algorithm outputs a sequence $\set{x_t^{(i)}}_{i\in[m]}$.
\end{compactenum}
\end{definition}

Our analysis is motivated by the single-token complexity lower bound in \cite{even2023stochastic}, the possibility of information exchange between meeting tokens introduces a new difficulty, since information can propagate across distinct random-walk trajectories.

\begin{assumption}\label{assu:Lsmooth}
The function $f = \sum_{u \in \cV} \pi_u f_u$ is $L$-smooth.
\end{assumption}
We denote 
\begin{align}\label{eq:Rtq}
R(t,q)\triangleq(1-q)\left[\ln\left(\frac{1-q}{q}\right)+2+\frac{2}{tq}\right],
\end{align}
where $q = 1-\left(1-\frac{1}{N-1}\right)^m$. By Bernoulli's inequality, $q \leq m/(N-1)$ for all integers $m \geq 1$.

\begin{theorem}[m-token lower bound]\label{thm:mtoken}
Consider the class of token-based algorithms satisfying Definition~\ref{def:MarkovBB} on a complete graph with m tokens. Let Assumption~\ref{assu:Lsmooth} hold. Let $\Delta,B,L,\mu>0$, define $\kappa=L/\mu$, and fix $x_0$. Let $R(t, q)$ and $q$ be given in \eqref{eq:Rtq}. For each token $i$, the following lower bounds hold.

\begin{enumerate}
\item Non-convex $f$:
There exist functions $\set{f_u}_{u \in \cV}$ satisfying $f(x_0) - f(x^\star) \le \Delta$. After $T$ steps,
 \begin{align}\label{eq:nonconvexerror}
\|\nabla f(x_t)\|^2 &\geq\frac{L\Delta}{64t^2q^2}\left(1+\frac{R(t,q)}{2tq}\right)^{-2}\nonumber\\
&\geq\frac{L\Delta (N-1)^2}{64m^2t^2}\left(1+O\left(\frac{1}{t}\right)\right)^{-2}.
\end{align}

\item Convex $f$:
There exist functions $\set{f_u}_{u \in \cV}$ such that $f$ is convex and admits a minimizer $x^\star$ satisfying $\|x_0 - x^\star\| \le B$. After $T$ steps, 
\begin{align}\label{eq:convexerror}
f(x_t) - f(x^\star)&\ge\frac{3LB^2}{128 t^2q^2} \left( 1 + \frac{R(t, q) + 1}{2tq} \right)^{-2}\nonumber \\
&\geq \frac{3LB^2 (N-1)^2}{128m^2 t^2} \left( 1 + O\left(\frac{1}{t}\right) \right)^{-2}.
\end{align}

\item Strongly convex $f$:
There exist functions $\set{f_u}_{u \in \cV}$ such that $f$ is $\mu$-strongly convex and admits a minimizer $x^\star$ satisfying $\|x_0 - x^\star\| \le B$.
After $T$ steps,
\begin{align}\label{eq:sconvexerror}
f(x_t) - f(x^\star)&\ge\frac{1}{2}\frac{L}{\kappa} B^2 \exp\left( - \frac{8tq(1 + \frac{R(t, q))}{2tq}}{\sqrt{\kappa}} \right)\nonumber\\
&\geq\frac{1}{2}\frac{L}{\kappa} B^2 \exp\left( - \frac{8mt(1 + O(\frac{1}{t})}{(N-1)\sqrt{\kappa}} \right).
\end{align}
\end{enumerate}
\end{theorem}
\begin{proof}
The full proof is given in Appendix~\ref{Appex:mtoken}.
\end{proof}
\begin{remark}\label{remark:UPT}
In \eqref{eq:nonconvexerror}, \eqref{eq:convexerror}, and \eqref{eq:sconvexerror}, by setting the right-hand side of the bound equal to $\epsilon$ and solving for $t$, we obtain the corresponding expression for $T_\epsilon^{(i)}$.
\end{remark}
The discussions of Theorem~\ref{thm:mtoken} are as follows:
\begin{compactenum}[(i)]
\item Theorem~\ref{thm:mtoken} applies to both update policies considered in Section~\ref{sec:InfoEx}. Since the information to be exchanged (either gradients or models) must first be acquired by at least one token before it can be shared with others through meetings.
\item Theorem~\ref{thm:mtoken} suggests that performance may depend more critically on the ratio $N/m$ than on $N$ or $m$ alone.
\item Compared with the single-token lower bounds in~\cite{even2023stochastic}, the scaling factor in our bounds is reduced by a factor of $m^2$ for both non-convex and convex objectives, and exponentially in $m$ for strongly convex objectives. The constant factor in the iteration complexity is also improved by a factor of $2$.
\end{compactenum}

\subsection{Best-Token and Mean-Token Losses}\label{subsec:othermetrics}
In practical applications with multiple tokens, performance can be evaluated from two perspectives: the best-performing token and the average performance across all tokens. We refer to the corresponding error measures as the \textit{best-token loss} and the \textit{mean-token loss}, respectively. 

For the best-token criterion, in the non-convex case, define
\begin{align}\label{eq:nonconvexbesttoken}
T_\epsilon^{\text{best}} \triangleq \inf\left\{t \ge 0 : \min_{j\in[m]} \|\nabla f(x_t^{(j)})\|^2\le \epsilon\right\}.
\end{align}
For the convex and strongly convex cases, define
\begin{align}\label{eq:convexbesttoken}
T_\epsilon^{\text{best}} \triangleq \left\{t \ge 0 : \min_{j\in[m]}f(x_t^{(j)})-f(x^\star)\le \epsilon\right\}.
\end{align}

For the mean-token criterion, in the non-convex case, define
\begin{align}\label{eq:nonconvexmeantoken}
T_\epsilon^{\text{mean}} \triangleq \inf\left\{t \ge 0 : \frac{1}{m}\sum_{j=1}^{m}\|\nabla f(x_t^{(j)})\|^2\le \epsilon\right\}.
\end{align}
For the convex and strongly convex cases, define
\begin{align}\label{eq:convexmeantoken}
T_\epsilon^{\text{mean}} \triangleq \left\{t \ge 0 : \frac{1}{m}\sum_{j=1}^{m}f(x_t^{(j)})-f(x^\star)\le \epsilon\right\}.
\end{align}

We next show how that the preceding results for the individual-token criterion in \eqref{eq:nonconvexTepsilon} and  \eqref{eq:convexTepsilon} extend to these two performance criteria. 
\begin{corollary}\label{cor:ExtendT}
The error expressions in Theorem~\ref{thm:mtoken} and the corresponding iteration complexity bounds described in Remark~\ref{remark:UPT}  also apply to $T_\epsilon^{\text{best}}$ and $T_\epsilon^{\text{mean}}$.
\end{corollary}

\begin{proof}
The proof is given in Appendix~\ref{Appe:ExtendT}.
\end{proof}

\section{Numerical Results}
In this section, we numerically investigate how their empirical performance depends on the network-to-token ratio, i.e., $\frac{N}{m}$, plays an important role in the convergence behavior. Following Section~\ref{subsec:othermetrics}, we evaluate the two policies using both the best-token and mean-token losses, respectively.

\subsection{Simulation Setup}
\paragraph{Graph Settings}
We consider a complete graph with $N$ nodes and $m$ tokens. We set the target distribution $\pi$ to be uniform. Based on the Metropolis-Hastings
algorithm \cite{hastings1970monte}, at each iteration, each token independently moves to a uniformly selected neighboring node. The initial token positions are sampled independently and uniformly over the nodes. Our reference setting is $N=100$ and $m=50$. To examine the effect of the network-to-token ratio $\frac{N}{m}$,  we vary $m\in\{10,50,100,200\}$ while fixing $N=100$. 

\paragraph{Learning Settings}
We consider a linear regression problem with $500$ synthetic samples and parameter dimension $d=10$. Each sample $(a_s,b_s)$ is generated as $a_s\sim\mathcal{N}(0,5I_d)$ and $b_s=a_s^\top x^\star+\epsilon_s$, where $\epsilon_s\sim\mathcal{N}(0,0.1^2)$ and the entries of  $x^\star$ are drawn independently from $\mathcal{N}(0,25)$.
The samples are randomly partitioned into $N$ disjoint subsets of approximately equal size under an \iid setting, with node $u$ holding dataset $\cD_u$ and local objective $f_u(x)=\frac{1}{|\cD_u|}\sum_{s\in\cD_u}(a_s^\top x-b_s)^2$.
At each iteration, each token performs one SGD update with a step size $\eta=5\times10^{-4}$, using a single sample drawn uniformly from the dataset at its current node. Tokens located at the same node use the same sampled data point. All tokens are initialized with the same model, whose entries are drawn independently from $\mathcal{N}(0,1)$. Each experiment runs for $1000$ iterations, and the results are averaged over $20$ trials.

\paragraph{Measures and Baseline}
Since the local objectives are convex, we evaluate the learning performance using the best-token and mean-token error measures defined in \eqref{eq:convexbesttoken} and \eqref{eq:convexmeantoken}, respectively.

We use single-token random-walk learning as the baseline. Without information exchange, each token in the $m$-token system follows the same learning dynamics as a single random-walk token. This baseline allows us to assess the performance gains achieved by multiple tokens with information exchange.

\subsection{Comparison of Exchange Policies}
We first compare the two exchange policies, \textsc{Average-Gradient} and \textsc{Average-Model}, for fixed $N$ and $m$. 
Fig.~\ref{fig:policy-comparison} reports the percentage improvement of each policy over the baseline in terms of both the mean-token and best-token losses, defined as
\begin{equation*}
\text{Improvement}(t) = \frac{\mathcal{L}_{\text{base}}(t) - \mathcal{L}_{\text{policy}}(t)}{\mathcal{L}_{\text{base}}(t)} \times 100\%,
\end{equation*}
where $\mathcal{L}_{\text{base}}(t)$ and $\mathcal{L}_{\text{policy}}(t)$ denote the losses of the baseline and the policy at iteration $t$, respectively. A positive value indicates that the policy achieves a lower loss than the baseline. The $x$-axis denotes the number of iterations. 
The two metrics lead to different conclusions: \textsc{Average-Model} performs better in terms of the mean-token loss, whereas \textsc{Average-Gradient} performs better in terms of the best-token loss.

\begin{figure}[htbp]
\centering
\begin{subfigure}[t]{0.49\linewidth}
\centering
\includegraphics[width=\linewidth]{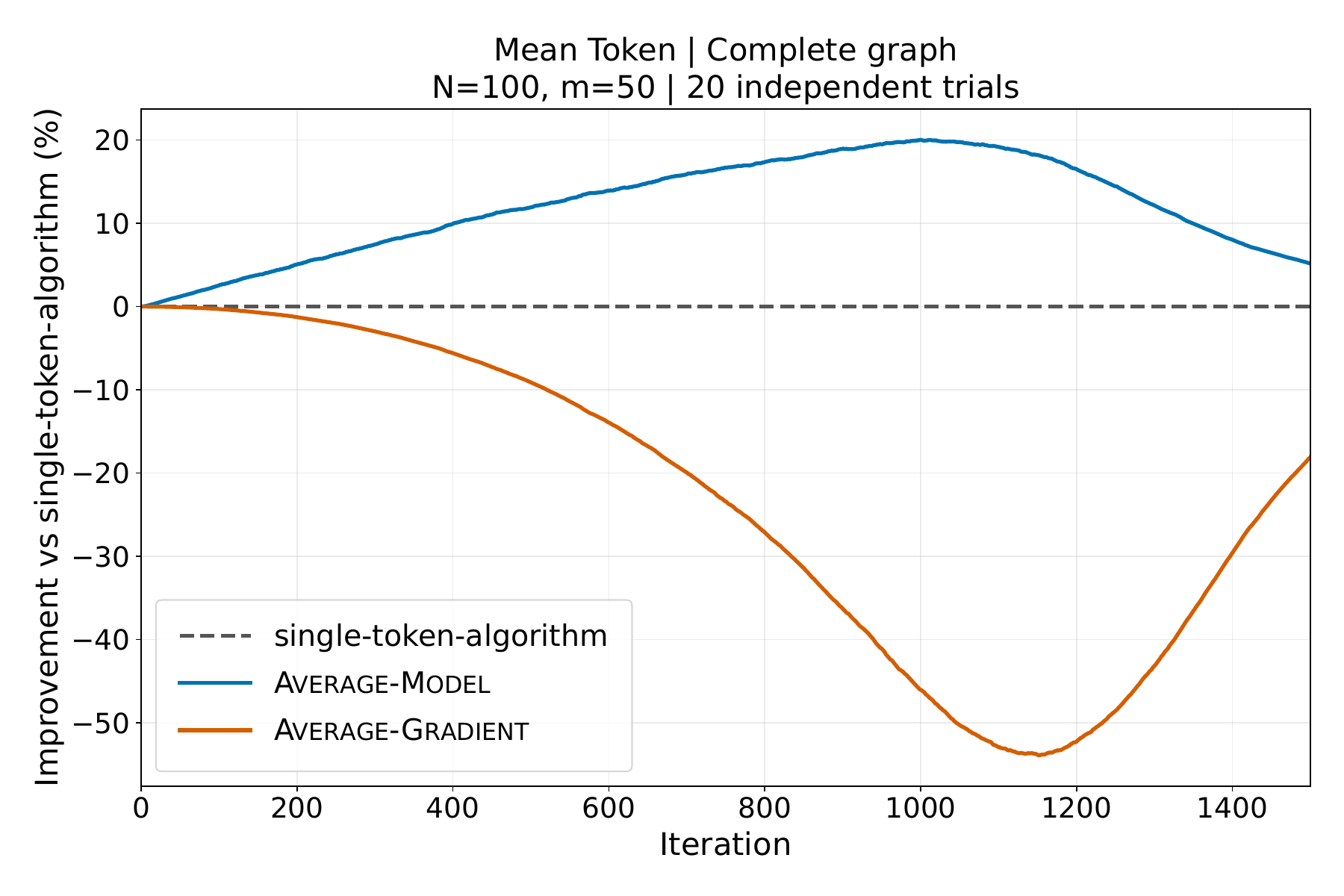}
\caption{Mean-token loss improvement}
\label{fig:policy-mean}
\end{subfigure}
\begin{subfigure}[t]{0.49\linewidth}
\centering
\includegraphics[width=\linewidth]{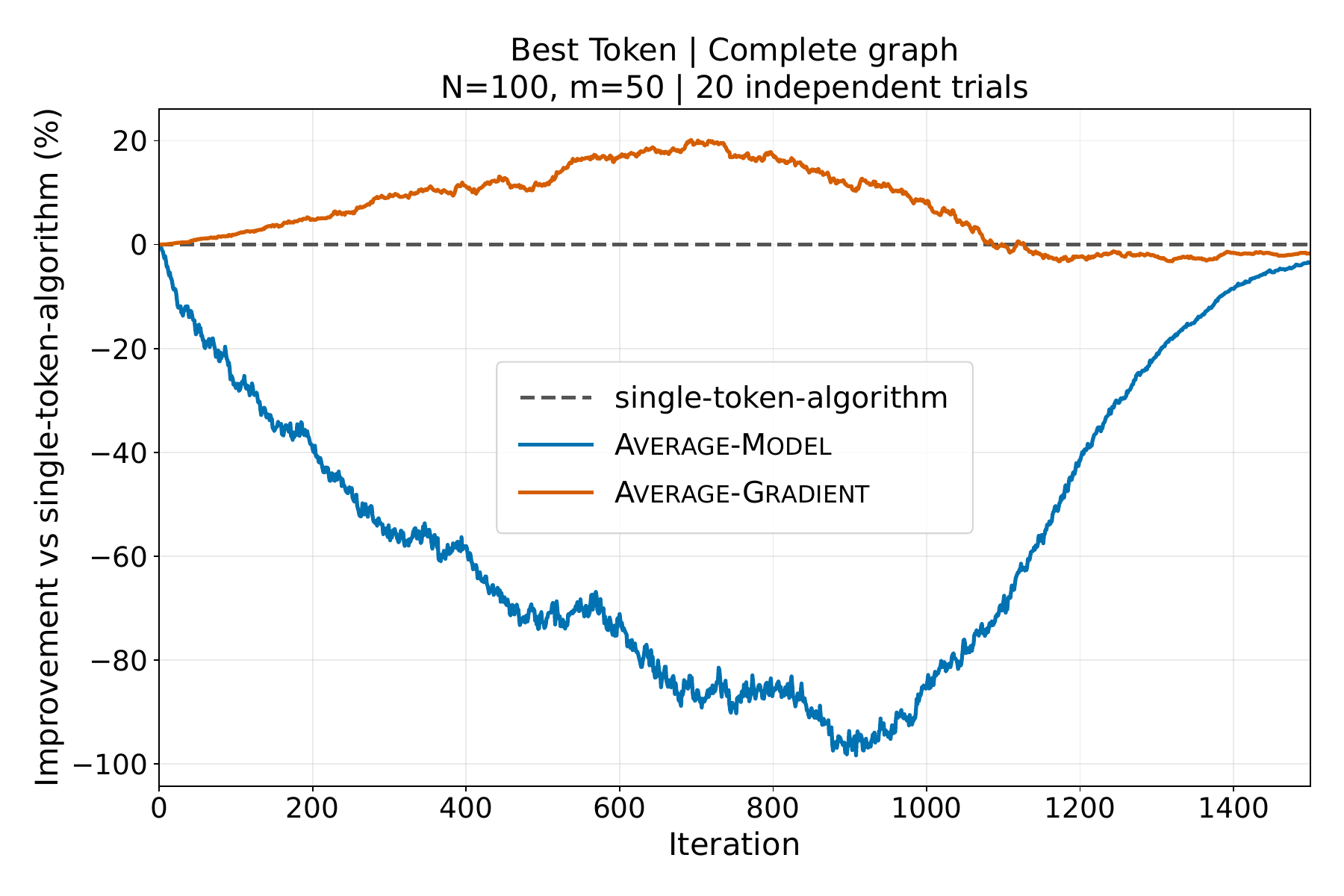}
\caption{Best-token loss improvement}
\label{fig:policy-best}
\end{subfigure}
\caption{Percentage improvement relative to single-token-algoritm for $N=100$ and $m=50$.}
\label{fig:policy-comparison}
\end{figure}%

As shown in Fig.~\ref{fig:policy-comparison}(\subref{fig:policy-mean}), \textsc{Average-Model} achieves a higher improvement than \textsc{Average-Gradient}. One possible explanation is that each token accumulates information along its own random-walk trajectory, and model averaging directly combines the information learned by different tokens. In contrast, \textsc{Average-Gradient} exchanges only the gradients computed at meetings, without directly combining the models accumulated along different trajectories. This difference may explain the better mean-token performance of \textsc{Average-Model}.

The opposite trend is observed in Fig.~\ref{fig:policy-comparison}(\subref{fig:policy-best}), where \textsc{Average-Gradient} achieves a higher improvement than \textsc{Average-Model}. Under \textsc{Average-Model}, a well-performing token may lose some of its accumulated advantage when its model is averaged with a worse-performing token. In contrast, \textsc{Average-Gradient} preserves each token's model while incorporating gradient information from other tokens. This allows a well-performing token to benefit from information exchange without directly replacing its current model, thereby better preserving the advantage of well-performing tokens.

\subsection{Effect of the Network-to-Token Ratio $\frac{N}{m}$}
We next study the effect of the network-to-token ratio $N/m$. As $m$ increases, $N/m$ decreases, leading to more frequent token meetings and information exchange. This allows us to examine how the two exchange policies perform under different levels of token interaction.

\paragraph{\textsc{Average-Gradient}} 
Fig.~\ref{fig:M-Average-gradient} shows that the effect of $N/m$ on \textsc{Average-Gradient} differs between the mean-token and best-token metrics.

For the mean-token loss \big(see Fig.~\ref{fig:M-Average-gradient}({\subref{fig:M-Average-gradient-mean}})\big), performance deteriorates as $N/m$ decreases, i.e., as the number of tokens increases. A smaller $N/m$ leads to more frequent meetings and gradient averaging. However, tokens generally carry different models due to their distinct random-walk trajectories, so their gradients are evaluated at different models. Averaging these gradients more frequently may perturb individual descent directions rather than improve them, potentially explaining the higher mean-token loss at smaller $\frac{N}{m}$.

\begin{figure}[htbp]
\centering
\begin{subfigure}[t]{0.49\linewidth}
\centering
\includegraphics[width=\linewidth]{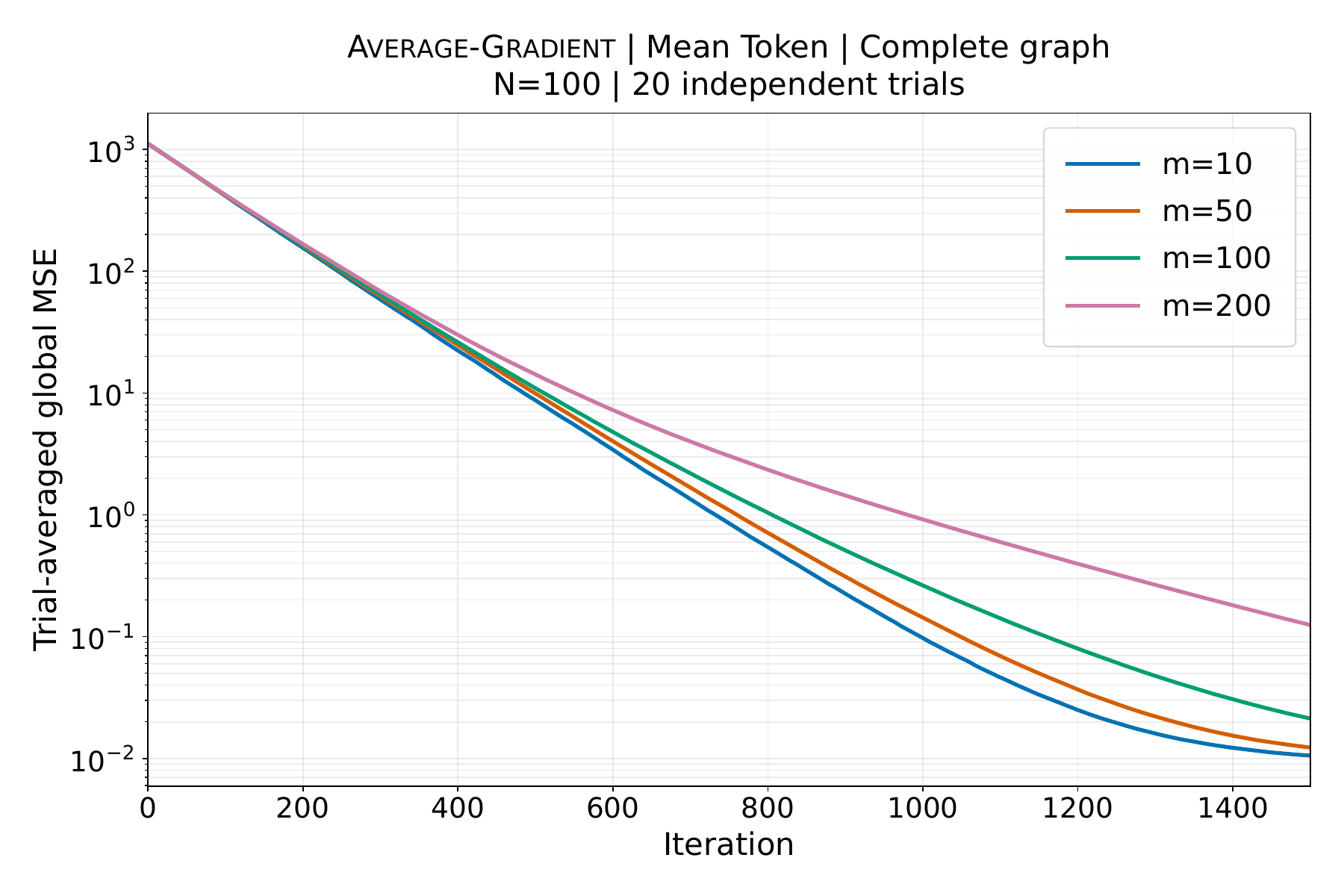}
\caption{Mean-token loss}
\label{fig:M-Average-gradient-mean}
\end{subfigure}
\begin{subfigure}[t]{0.49\linewidth}
\centering
\includegraphics[width=\linewidth]{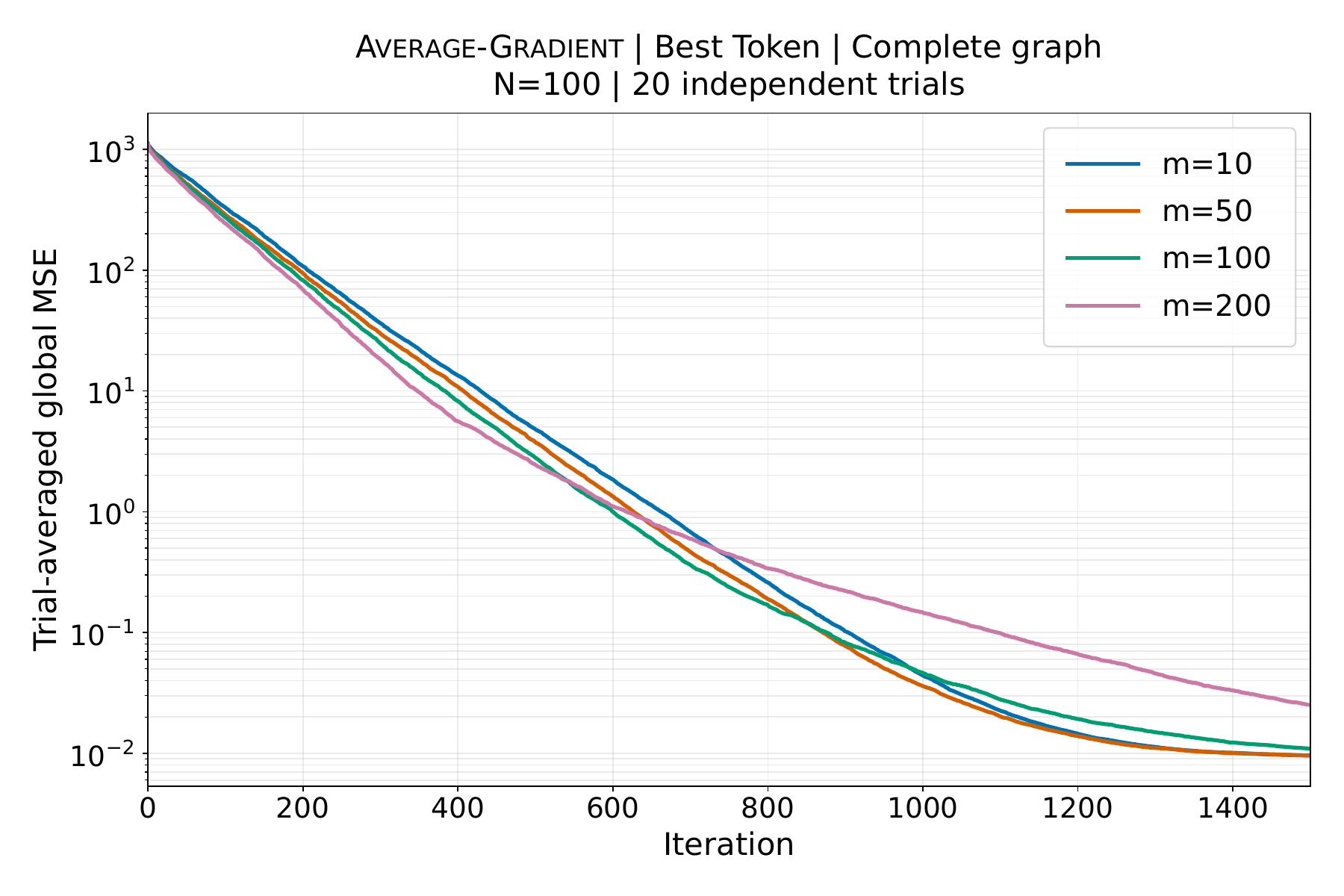}
\caption{Best-token loss}
\label{fig:M-Average-gradient-best}
\end{subfigure}
\caption{Mean-token and best-token global MSE under \textsc{Average-Gradient} for different values of $m$, averaged over $20$ trials.}
\label{fig:M-Average-gradient}
\end{figure}

The best-token loss \big(Fig.~\ref{fig:M-Average-gradient}(\subref{fig:M-Average-gradient-best})\big) exhibits a different two-stage behavior. During the early iterations, decreasing $N/m$ improves best-token performance, as more tokens explore the network in parallel, increasing the likelihood that at least one token achieves a lower loss.

However, this advantage weakens later in training. For $m = 200$, the decrease in best-token loss begins to slow at approximately iteration $600$, while a similar slowdown occurs later for $m = 100$. One possible explanation is that, as well-performing tokens approach a good solution, gradients contributed by tokens at different model states may perturb their descent directions. With more frequent meetings at smaller $\frac{N}{m}$, such interference may arise earlier.

These results suggest a trade-off: decreasing $N/m$ enhances early best-token performance through parallel exploration but may degrade mean-token performance and weaken the best-token advantage later in training due to more frequent gradient averaging.

\paragraph{\textsc{Average-Model}}

A different pattern is observed for \textsc{Average-Model}. As shown in Fig.~\ref{fig:M-Average-model}, both the mean-token and best-token losses remain nearly unchanged as $N/m$ decreases. Unlike \textsc{Average-Gradient}, \textsc{Average-Model} directly combines the models accumulated along different token trajectories at meetings. Once meetings become sufficiently frequent to propagate information across the token population, additional meetings provide limited benefit. In our experiments, even $m = 10$ tokens provide sufficient opportunities for model exchange, so increasing $m$ further yields little additional benefit.

These results suggest that the two policies differ in their sensitivity to $N/m$: \textsc{Average-Gradient} exhibits substantial changes in learning performance as the ratio decreases, whereas \textsc{Average-Model} remains relatively stable over the tested range.

\begin{figure}[htbp]
\centering
\begin{subfigure}[t]{0.49\linewidth}
\centering
\includegraphics[width=\linewidth]{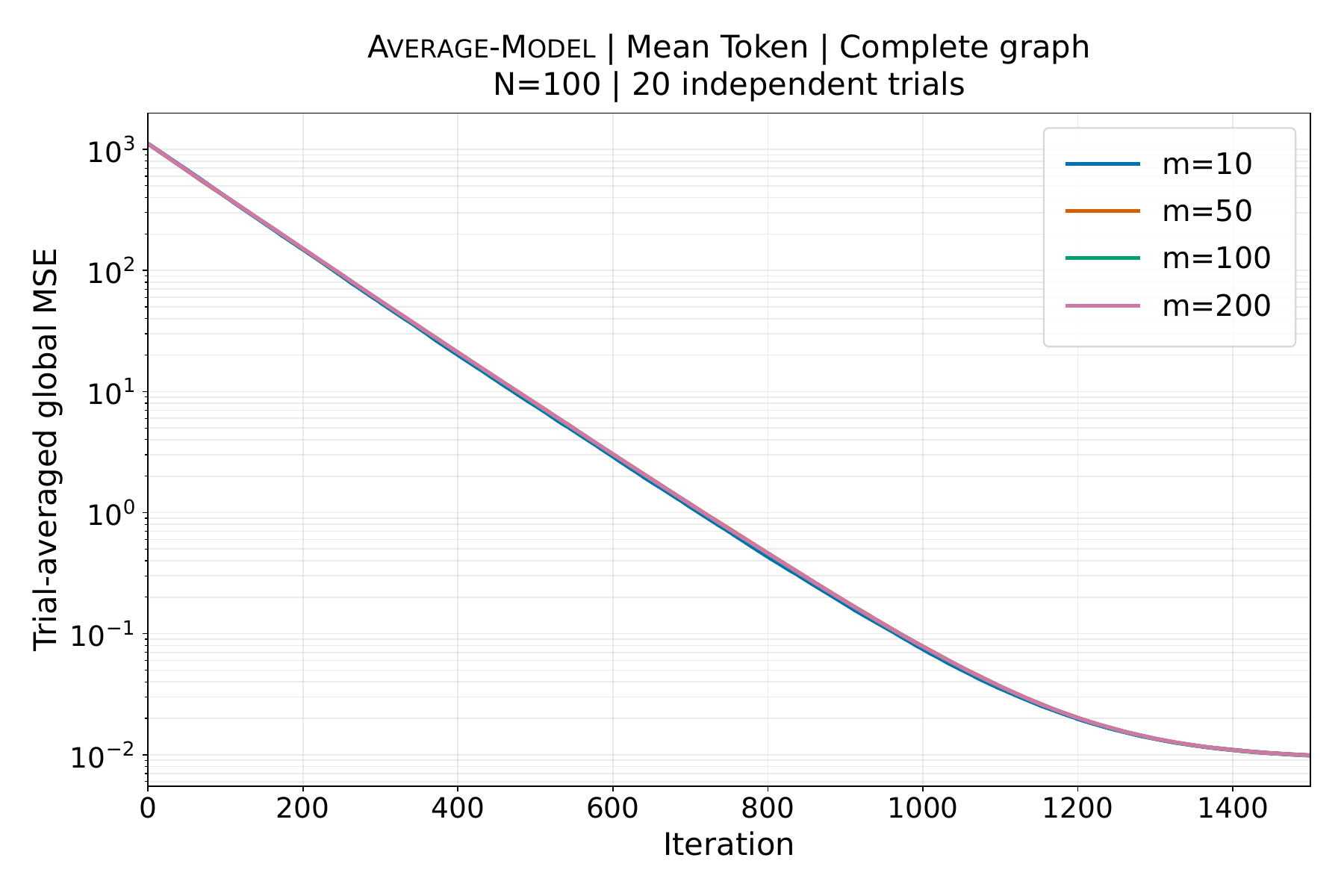}
\caption{Mean-token loss}
\label{fig:M-Average-model-mean}
\end{subfigure}
\begin{subfigure}[t]{0.48\linewidth}
\centering
\includegraphics[width=\linewidth]{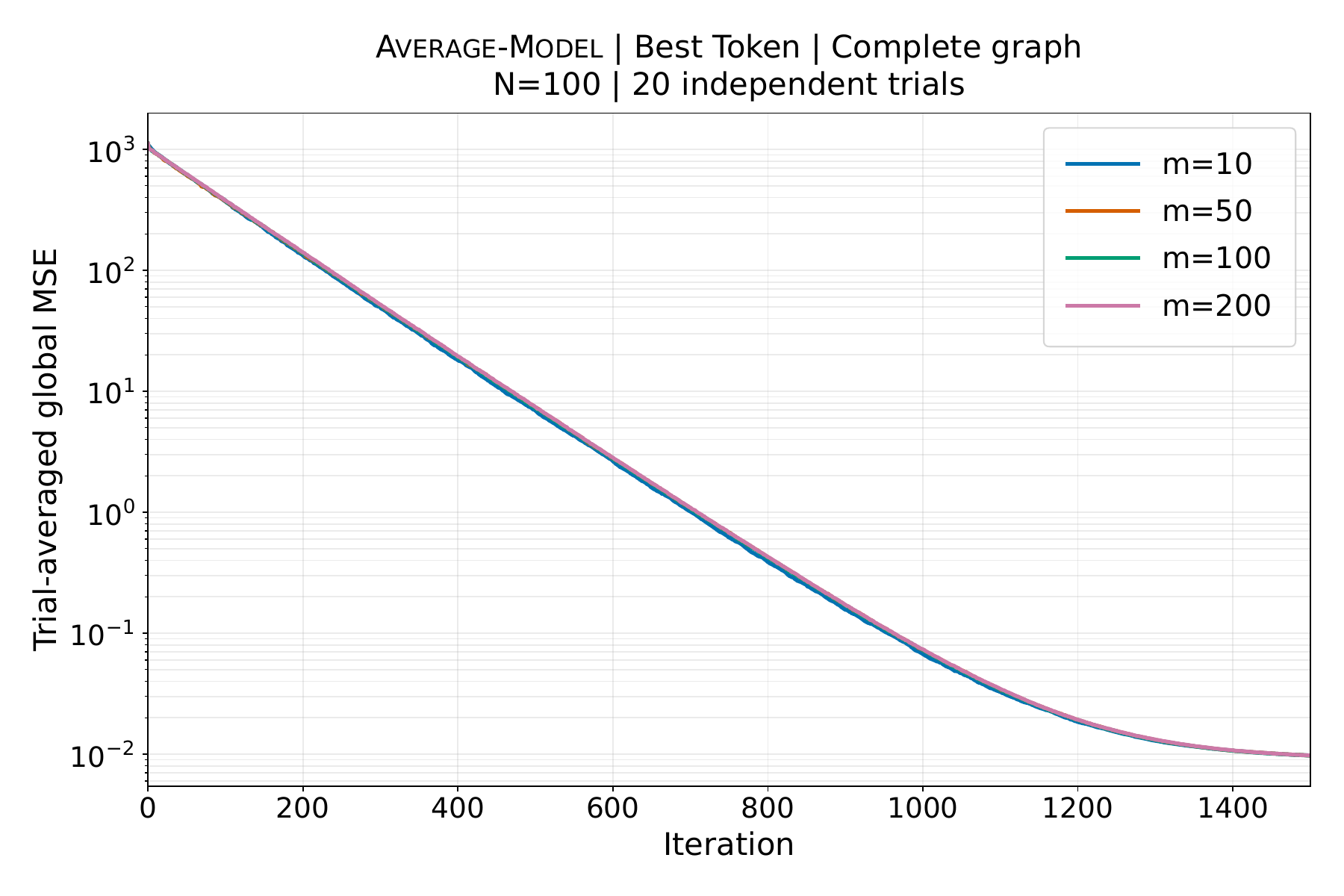}\caption{Best-token loss}
\label{fig:M-Average-model-best}
\end{subfigure}
\caption{Mean-token and best-token global MSE under \textsc{Average-Model} for different values of $m$, averaged over $20$ trials.}
\label{fig:M-Average-model}
\end{figure}%

\section{Conclusion}
In this work, we investigated how meeting-triggered information exchange affects multi-token RW learning. Our results yield two main findings. First, on complete graphs, the iteration complexity lower bounds for interacting multi-token algorithms under the Markov black-box framework scale approximately with the network-to-token ratio $N/m$. Our experiments further show that \textsc{Average-Gradient} is more sensitive to this ratio than \textsc{Average-Model}. Second, the preferred information exchange policy depends on the performance criterion: \textsc{Average-Model} performs better when evaluated across all tokens, whereas \textsc{Average-Gradient} performs better when evaluated by the best-performing token. Future research could extend our theoretical results from complete graphs to general network topologies and establish theoretical guarantees explaining the observed advantages of \textsc{Average-Model} in mean-token performance and \textsc{Average-Gradient} in best-token performance.

\clearpage

\bibliographystyle{IEEEtran}
\bibliography{reference}

@article{chen2026selfcreating,
  title   = {Self-Creating Random Walks for Decentralized Learning
             under Pac-Man Attacks},
  author  = {Chen, Xingran and Parag, P. and Bhagat, R.
             and El Rouayheb, S.},
  journal = {IEEE Journal on Selected Areas in Communications},
  year    = {2026},
  note    = {accepted},
}

@inproceedings{chen2026pacman,
  title     = {Random Walk Learning and the Pac-Man Attack},
  author    = {Chen, Xingran and Parag, P. and Bhagat, R.
               and Liu, Z. and El Rouayheb, S.},
  booktitle = {2026 IEEE International Symposium on Information Theory},
  year      = {2026}
}

@inproceedings{chen2026rwlora,
  title     = {{RW-LoRA}: Communication-Efficient Decentralized
               {LoRA} Fine-Tuning via Random Walks},
  author    = {Chen, Xingran and Bhagat, R. and Ayache, G.
               and Bitar, R. and Gong, Y. and El Rouayheb, S.},
  booktitle = {2026 IEEE Information Theory Workshop},
  year      = {2026}
}

@inproceedings{koloskova2020unified,
  title={A unified theory of decentralized SGD with changing topology and local updates},
  author={Koloskova, Anastasia and Loizou, Nicolas and Boreiri, Sadra and Jaggi, Martin and Stich, Sebastian},
  booktitle={International conference on machine learning},
  pages={5381--5393},
  year={2020},
  organization={PMLR}
}

@inproceedings{even2023stochastic,
  title={Stochastic gradient descent under Markovian sampling schemes},
  author={Even, Mathieu},
  booktitle={International conference on machine learning},
  pages={9412--9439},
  year={2023},
  organization={PMLR}
}

@book{nesterov2018lectures,
  title     = {Lectures on Convex Optimization},
  author    = {Nesterov, Yurii},
  year      = {2018},
  publisher = {Springer},
  address   = {Cham},
  series    = {Springer Optimization and Its Applications},
  volume    = {137},
  doi       = {10.1007/978-3-319-91578-4}
}

@article{carmon2021lower,
  title   = {Lower Bounds for Finding Stationary Points II: First-Order Methods},
  author  = {Carmon, Yair and Duchi, John C. and Hinder, Oliver and Sidford, Aaron},
  journal = {Mathematical Programming},
  volume  = {185},
  pages   = {315--355},
  year    = {2021},
  doi     = {10.1007/s10107-019-01431-x}
}

@collection{lovasz1993random,
  author    = {Lov{\'a}sz, L{\'a}szl{\'o}},
  title     = {Random Walks on Graphs},
  booktitle = {Combinatorics, Paul Erd{\H{o}}s Is Eighty},
  volume    = {2},
  pages     = {1--46},
  year      = {1993}
}

@book{levin2017markov,
  author    = {Levin, David A. and Peres, Yuval},
  title     = {Markov Chains and Mixing Times},
  edition   = {2},
  publisher = {American Mathematical Society},
  year      = {2017}
}

@techreport{page1999pagerank,
  author      = {Page, Lawrence and Brin, Sergey and
                 Motwani, Rajeev and Winograd, Terry},
  title       = {The {PageRank} Citation Ranking:
                 Bringing Order to the Web},
  institution = {Stanford InfoLab},
  year        = {1999}
}

@article{fouss2007random,
  author  = {Fouss, Fran{\c{c}}ois and Pirotte, Alain and
             Renders, Jean-Michel and Saerens, Marco},
  title   = {Random-Walk Computation of Similarities between
             Nodes of a Graph with Application to
             Collaborative Recommendation},
  journal = {IEEE Transactions on Knowledge and Data Engineering},
  volume  = {19},
  number  = {3},
  pages   = {355--368},
  year    = {2007}
}

@inproceedings{backstrom2011supervised,
  author    = {Backstrom, Lars and Leskovec, Jure},
  title     = {Supervised Random Walks: Predicting and
               Recommending Links in Social Networks},
  booktitle = {Proceedings of the Fourth ACM International
               Conference on Web Search and Data Mining},
  pages     = {635--644},
  year      = {2011}
}

@inproceedings{johansson2007peer,
  author    = {Johansson, B. and Rabi, M. and Johansson, M.},
  title     = {A Simple Peer-to-Peer Algorithm for Distributed
               Optimization in Sensor Networks},
  booktitle = {46th IEEE Conference on Decision and Control},
  pages     = {4705--4710},
  year      = {2007}
}

@article{ayache2023walk,
  author  = {Ayache, G. and Dassari, V. and El Rouayheb, S.},
  title   = {Walk for Learning: A Random Walk Approach for
             Federated Learning From Heterogeneous Data},
  journal = {IEEE Journal on Selected Areas in Communications},
  volume  = {41},
  number  = {4},
  pages   = {929--940},
  year    = {2023}
}

@article{ye2020parallel,
  author  = {Ye, Yu and Chen, Hao and Ma, Zheng and Xiao, Ming},
  title   = {Decentralized Consensus Optimization Based on
             Parallel Random Walk},
  journal = {IEEE Communications Letters},
  volume  = {24},
  number  = {2},
  pages   = {391--395},
  year    = {2020},
  doi     = {10.1109/LCOMM.2019.2955442}
}

@inproceedings{hendrikx2023token,
  author    = {Hendrikx, Hadrien},
  title     = {A Principled Framework for the Design and
               Analysis of Token Algorithms},
  booktitle = {Proceedings of the 26th International Conference
               on Artificial Intelligence and Statistics},
  series    = {Proceedings of Machine Learning Research},
  volume    = {206},
  pages     = {470--489},
  publisher = {PMLR},
  year      = {2023}
}

@article{valdeira2026multitoken,
  author  = {Valdeira, Pedro and Chi, Yuejie and
             Soares, Cl{\'a}udia and Xavier, Jo{\~a}o},
  title   = {A Multi-Token Coordinate Descent Method for
             Semi-Decentralized Vertical Federated Learning},
  journal = {IEEE Transactions on Signal Processing},
  volume  = {74},
  pages   = {293--308},
  year    = {2026}
}

@article{mao2020walkman,
  author  = {Mao, Xianghui and Yuan, Kun and Hu, Yubin and
             Gu, Yuantao and Sayed, Ali H. and Yin, Wotao},
  title   = {Walkman: A Communication-Efficient Random-Walk
             Algorithm for Decentralized Optimization},
  journal = {IEEE Transactions on Signal Processing},
  volume  = {68},
  pages   = {2513--2528},
  year    = {2020}
}

@article{hastings1970monte,
  author  = {Hastings, W. K.},
  title   = {Monte Carlo sampling methods using {Markov} chains and their applications},
  journal = {Biometrika},
  volume  = {57},
  number  = {1},
  pages   = {97--109},
  year    = {1970},
  doi     = {10.1093/biomet/57.1.97}
}

\onecolumn
\appendices

\section{Proof for Theorem~\ref{thm:mtoken}}\label{Appex:mtoken}
In this proof, we first consider the case $m=2$ to illustrate the main idea. The argument then extends straightforwardly to the general case of $m$ tokens.

\subsection{Sequential Coordinate Revelation}
Let $u,w\in\cV$ be two designated nodes. Node $u$ is associated with the odd-indexed coordinates, whereas node $w$ is associated with the even-indexed coordinates. The coordinates must be revealed sequentially: coordinate $k$ can be revealed only after coordinates $1,\ldots,k-1$ have already been revealed. 
Thus, after coordinate $k-1$ has been revealed, token $i$ must next visit
\begin{align*}
r_k \coloneqq
\begin{cases}
u, & k \text{ is even},\\
w, & k \text{ is odd},
\end{cases}
\end{align*}
in order to reveal coordinate $k$.  In other words, the token must successively visit $u,w,u,w,\ldots$ to reveal coordinates $1,2,3,4,\ldots$, respectively.

We define 
\begin{align*}
\mathcal{H}_k(t)\coloneqq\left\{j :\text{token } j \text{ has acquired coordinates }1,\ldots,k\text{ by time }t\right\}.
\end{align*}
For each token $i$, let $T_k^{(i)}$ denote the first time at which token $i$ has access to the first $k$ coordinates. Set  $T_{0}^{(i)} \coloneqq 0$.
\begin{align*}
T_k^{(i)}\coloneqq\inf\left\{t \ge T_{k-1}^{(i)}: v_t^{(i)} = r_k\text{ or }\exists j \in \cH_k(t)\text{ such that }v_t^{(j)} = v_t^{(i)}\right\},
\quad k \ge 1.
\end{align*}
Based on the definition of $T_k^{(i)}$, we denote
\begin{equation}
k^{(i)}(t) \coloneqq \sup \bigl\{ k \ge 0 : T_k^{(i)} \le t \bigr\}.
\end{equation}
By the definition of $k^{(i)}(t)$, we have $k^{(i)}(t)\le t$. In addition, $\{k^{(i)}(t) \ge k\} \Leftrightarrow \{T_k^{(i)} \le t\}$.

\begin{align*}
\mathbb{E}\bigl[k^{(i)}(t)\bigr] = \sum_{0 \le k \le t} \mathbb{P}\bigl(k^{(i)}(t) \ge k\bigr) = \sum_{0 \le k \le t} \mathbb{P}(T_k^{(i)} \le t).
\end{align*}

\subsection{Bounding Acquisition-Time Probabilities Using the Minimum}\label{app:acquisition-time-bound}

We first consider the time required for either of the two tokens to have access to successive coordinates. Define
\begin{align*}
S_\ell \coloneqq\min\left\{T_\ell^{(i)},T_\ell^{(j)}\right\},
\end{align*}
which is the first time at which either token has access to the first $\ell$ coordinates. We set $S_0\coloneqq 0$. Then,
\begin{align}\label{eq:UBTk}
\Pr\bigl(T_k^{(i)} \le t\bigr)\le\Pr\bigl(S_k \le t\bigr)=\Pr\Biggl(\sum_{\ell=1}^{k}\bigl(S_\ell-S_{\ell-1}\bigr)\le t\Biggr)=\Pr\Biggl(\sum_{\ell=1}^{k}Z_\ell(r_\ell,X_{\ell-1})\le t\Biggr),
\end{align}
where
\begin{align*}
Z_\ell(r_\ell,X_{\ell-1})\coloneqq S_\ell-S_{\ell-1}.
\end{align*}
Here, $r_\ell \in \cV$ denotes the node associated with the $\ell$-th coordinate. At time $S_\ell$, one of the two tokens is the first to have access to the $\ell$-th coordinate. We let $X_\ell\in\cV$ denote the position of the other token at that time. Specifically,
\begin{align*}
X_\ell \coloneqq
\begin{cases}
v_{S_\ell}^{(j)},& T_{\ell}^{(i)} \le T_{\ell}^{(j)},\\[1mm]
v_{S_\ell}^{(i)},& T_{\ell}^{(j)} < T_{\ell}^{(i)}.
\end{cases}
\end{align*}

\subsection{Bounding the Waiting Time $Z_\ell$ Using $G_\ell$}
At the beginning of stage $\ell$, the two tokens are located at $r_{\ell-1}$ and $X_{\ell-1}$. By the law of total probability,
\begin{align*}
\Pr(Z_\ell>s)&=\Pr(Z_\ell>s\mid X_{\ell-1}\neq r_\ell)\Pr(X_{\ell-1}\neq r_\ell)\\
&+\Pr(Z_\ell>s\mid X_{\ell-1}=r_\ell)\Pr(X_{\ell-1}=r_\ell).
\end{align*}

Let $G_\ell=\Pr(Z_\ell>s\mid X_{\ell-1}\neq r_\ell)$, we then have 
\begin{equation}\label{eq:no-start-on-target}
\Pr(Z_\ell>s)\geq\Pr(G_\ell>s).
\end{equation}
In a complete graph, by the Metropolis-Hastings algorithm \cite{hastings1970monte}, the transition probability matrix is uniform. From any non-target node $x \neq r_\ell$, a token reaches the target node $r_\ell$ in one step with probability $p=\frac{1}{N-1}$. Since the two tokens move independently, the probability that at least one token reaches $r_\ell$ in a given step is
\begin{align*}
q = 1-(1-p)^2=1-\left(1-\frac{1}{N-1}\right)^2.
\end{align*}
Hence, for any realization $X_\ell=x$, $G_\ell(r_\ell,x)\sim \operatorname{Geom}(q)$. It follows that 
\begin{align*}
\Pr\bigl(G_\ell(r_\ell,x)>s\bigr)=(1-q)^s.
\end{align*}

\noindent
For every realization $\mathbf{x}$, averaging over all possible trajectories gives 
\begin{align*}
\Pr\left(G_\ell(r_\ell,X_\ell)>s\right)&=\sum_{\mathbf{x}}\Pr\left(G_\ell>s\mid X_{\ell}=\mathbf{x}\right)\Pr\left(X_{\ell}=\mathbf{x}\right)\\
&=\sum_{\mathbf{x}}(1-q)^s\Pr\left(X_\ell=\mathbf{x}\right)\\
&=(1-q)^s.
\end{align*}
Moreover, since the same parameter $q$ applies to every stage $\ell$,
the random variables $G_\ell$ are identically distributed for all $\ell$.
Equivalently,
\begin{align*}
G_0 \overset{d}{=} G_1 \overset{d}{=} \cdots \sim \operatorname{Geom}(q).
\end{align*}
\subsection{Independent and Identically Distributed Waiting Times $G_\ell$}\label{app:two-token-independence}
To further establish independence, Let
\begin{align*}
\mathcal{F}_t\coloneqq\sigma\!\left(v_s^{(i)},v_s^{(j)}:0\le s\le t \right),
\end{align*}
where $v_s^{(i)}$ and $v_s^{(j)}$ denote the positions of the two tokens at
time $s$.
By the strong Markov property, conditioned on
$\mathcal F_{\ell}$, the evolution after $\ell$ depends only on the
current configuration $(r_\ell,X_\ell)$.
Hence,
\[
\mathbb P\left(
G_{\ell+1}>s
\mid
\mathcal F_{\ell}
\right)
=
\mathbb P_{}
\left(
G_{\ell+1}>s \mid  r_\ell,X_\ell
\right).
\]
For the complete graph, the right-hand side is independent
of $X_\ell$ and, by the argument above, we then have 
\[
\mathbb P_{}
\left(
G_{\ell+1}>s \mid  r_\ell,X_\ell
\right)
=
(1-q)^s
=
\mathbb{P}\left(
G_{\ell+1}>s
\right).
\]
Thus, $G_{\ell+1}$ is independent of $\mathcal{F}_{S_\ell}$. S
we have
\[
G_{\ell+1}\perp (G_1,\ldots,G_\ell).
\]
Hence, together with the identical distribution established above,
\[
G_\ell \overset{\mathrm{i.i.d.}}{\sim} \operatorname{Geom}(q).
\]
\subsection{Bound on the Number of Exchanges}
Next, we use Chebyshev's inequality to derive a probability bound. Since
$G_\ell \sim \operatorname{Geom}(q)$, we have
\[
\mathbb{E}[G_\ell]=\frac{1}{q},
\qquad
\operatorname{Var}(G_\ell)=\frac{1-q}{q^2}.
\]

Since $\{G_\ell\}$ are i.i.d.,
\[
\mathbb{E}\left[\sum_{\ell=1}^{k} G_\ell\right]
=
k\mathbb{E}[G_1],
\qquad
\operatorname{Var}\left(\sum_{\ell=1}^{k} G_\ell\right)
=
k\operatorname{Var}(G_1).
\]

Therefore, for $k\mathbb{E}[G_1]>t$,
\begin{align}
\mathbb{P}\left(
\sum_{\ell=1}^{k} G_\ell \le t
\right)
&=
\mathbb{P}\left(
\sum_{\ell=1}^{k} G_\ell
-
k\mathbb{E}[G_1]
\le
t-k\mathbb{E}[G_1]
\right)
\notag\\
&\le
\mathbb{P}\left(
\left|
\sum_{\ell=1}^{k} G_\ell
-
k\mathbb{E}[G_1]
\right|
\ge
k\mathbb{E}[G_1]-t
\right)
\notag\\
&\le
\frac{
k\operatorname{Var}(G_1)
}{
\left(
k\mathbb{E}[G_1]-t
\right)^2
}.
\end{align}

We then have:
\begin{align}\label{eq: bound0}
    \mathbb{E}[k^{(i)}(t)] &\leq \sum_{0\leq k\leq t}\mathbb{P}\big(\sum_{\ell=1}^{k} G_\ell \le t \big)\nonumber \\
    &= \sum_{0\leq k\leq 2tq}\mathbb{P}\big(\sum_{\ell=1}^{k} G_\ell \le t \big) + \sum_{2tq\leq k\leq t}\mathbb{P}\big(\sum_{\ell=1}^{k} G_\ell \le t \big)\nonumber\\
    &\leq2tq + \sum_{2tq\leq k\leq t} \frac{k\operatorname{Var}(G_1)}{\left(k\mathbb{E}[G_1]-t\right)^2}\nonumber\\
    &=2tq + \sum_{2tq\leq k\leq t}\frac{k(1-q)}{(k-tq)^2}. 
\end{align}

We split the sum at $2tq$ so that $k-tq \geq k/2$ in the tail sum, allowing us to bound the second term using Chebyshev's inequality.
We can further bound for second term in \eqref{eq: bound0}.
\begin{align}\label{eq:Bound1}
\sum_{2tq \le k \le t} \frac{k(1-q)}{(k - tq)^2} 
&= (1-q) \sum_{2tq \le k \le t} \frac{k}{(k - tq)^2}\nonumber \\
&= (1-q) \sum_{0 \le u \le t - 2tq} \frac{u + 2tq}{(u + tq)^2} \nonumber\\
&= (1-q) \left( \sum_{0 \le u \le t - 2tq} \frac{u}{(u + tq)^2} + \sum_{0 \le u \le t - 2tq} \frac{2tq}{(u + tq)^2} \right)
\end{align}

Let $u+tq = k-tq$. we then bound for first term in \eqref{eq:Bound1}.
\begin{alignat*}{2}
\sum_{0 \le u \le t - 2tq} \frac{u}{(u + tq)^2} &\le \sum_{0 \le u \le t} \frac{u}{(u + tq)^2} 
&&\le \sum_{1 \le u \le t} \frac{1}{u + tq} \\
&\le \int_{0}^{t} \frac{1}{u + tq} \, du 
&&= \ln(u + tq) \;\Big|_{0}^{t - 2tq} \\
&= \ln(t - tq) - \ln(tq) 
&&= \ln \frac{1 - q}{q}.
\end{alignat*}
The second inequality is because $u/(u+tq) \leq1$ and for $u=0$, we have zero value so we start from index $1$. 

Finally we bound for second term in \eqref{eq:Bound1}. 
\begin{alignat*}{2}
\sum \frac{2tq}{(u + tq)^2} 
&= 2 \sum \frac{tq}{(u + tq)^2} 
&\quad &\le 2 \sum_{u=0}^{\infty} \frac{tq}{(u + tq)^2} \\[1em]
&= 2 \biggl( \frac{1}{tq} + \sum_{u=1}^{\infty} \frac{tq}{(u + tq)^2} \biggr) 
&\quad &\le 2 \biggl( \frac{1}{tq} + \int_{0}^{\infty} \frac{tq}{(u + tq)^2} \, du \biggr) \\[1em]
&&&= 2 \biggl( \frac{1}{tq} + 1 \biggr)
\end{alignat*}

Combine all result we have, we can get:
\begin{align*}
\E[k^{(i)}(t)] 
&\le 2tq + (1-q) \left[ \ln \frac{1-q}{q} + \frac{2}{tq} + 2 \right]
\end{align*}

Let
$R(t,q)
=
(1-q)
\left[
\ln\left(\frac{1-q}{q}\right)
+2+\frac{2}{tq}
\right]$, we have 
\begin{align}\label{eq:UBk}
\E[k^{(i)}(t)]\le 2tq + R(t,q).
\end{align}

\subsection{Extension to $m$ Tokens}
The proof for multiple tokens follows the same argument as in the two-token case. We define
\begin{align*}
\alpha_\ell:= \min_{j\in[m]} \set{T_\ell^{(j)}}.
\end{align*}
Similar to \eqref{eq:UBTk}, we have
\begin{align}
\mathbb{P}\bigl(T_k^{(i)} \le t\bigr)&\le\mathbb{P}\bigl(\alpha_k \le t\bigr)=
\mathbb{P}\Biggl(
\sum_{\ell=1}^{k}
\bigl(\alpha_\ell-\alpha_{\ell-1}\bigr)
\le t
\Biggr) \notag\\
&=\mathbb{P}\Biggl(\sum_{\ell=1}^{k}\beta_\ell\le t\Biggr) \le \mathbb{P}\Biggl(\sum_{\ell=1}^{k}\gamma_\ell\le t\Biggr).
\end{align}
where
$\beta_\ell = \alpha_\ell-\alpha_{\ell-1}$. Following the same argument as in~\eqref{eq:no-start-on-target},
we define
\begin{align*}
A_\ell:=\bigcap_{i=1}^{m}\left\{v_{S_{\ell-1}}^{(i)}\neq r_\ell\right\},\qquad\gamma_\ell(s):=\Pr(\beta_\ell>s\mid A_\ell).
\end{align*}
By an argument similar to that for the two-token case in Appendix~\ref{app:two-token-independence}, the symmetry of the random walks implies that \(\{\gamma_\ell\}_{\ell=1}^k\) are independent and identically distributed. In particular,
\[\gamma_\ell\overset{\mathrm{i.i.d.}}{\sim}\operatorname{Geom}(q),\quad \text{where }q=1-\left(1-\frac{1}{N-1}\right)^m.
\]


\noindent
Therefore, following the same analysis in \eqref{eq:UBk}, we obtain
\begin{align*}
\mathbb{E}[k^{(i)}(t)] \le 2tq + R(t,q),
\end{align*}
where
\begin{align*}
R(t,q) = (1-q)\left[\ln\left(\frac{1-q}{q}\right) + 2 + \frac{2}{tq}\right].
\end{align*}
To express this bound in terms of $N$ and $m$, let $p = 1/(N-1)$. By Bernoulli's inequality, $(1-p)^m \geq 1-mp$, and hence
\begin{align*}
q = 1-(1-p)^m \leq mp = \frac{m}{N-1}
\end{align*}
holds for all $m \geq 1$. Substituting this into the leading term, we get
\begin{align*}
\mathbb{E}[k^{(i)}(t)] \le \frac{2tm}{N-1} + R(t,q).
\end{align*}
Since $R(t,q)$ is bounded in $t$ for fixed $N$ and $m$, the leading term $2tm/(N-1)$ dominates for large $t$.
For fixed $N$ and $m$, the correction term $R(t,q)$ is bounded with respect to $t$, i.e., $R(t,q) = O(1)$. Following~\cite{even2023stochastic}, the gradient for non-convex objectives satisfies $\mathbb{E}\!\left[\|\nabla f(x_t)\|^2\right] \geq \frac{L\Delta}{16\,\mathbb{E}[k^{(i)}(t)]^2}$. Substituting $\mathbb{E}[k^{(i)}(t)] \leq \frac{2tm}{N-1} + R(t,q)$, we obtain
\begin{align*}
\mathbb{E}\!\left[\|\nabla f(x_t)\|^2\right]
&\geq \frac{L\Delta}{16\left(\frac{2tm}{N-1} + R(t,q)\right)^2} \\
&= \frac{L\Delta(N-1)^2}{64m^2t^2}\left(1+\frac{(N-1)R(t,q)}{2tm}\right)^{-2} \\
&= \frac{L\Delta(N-1)^2}{64m^2t^2}\left(1+O\left(\frac{1}{t}\right)\right)^{-2}.
\end{align*}

Following a similar argument to the non-convex case, together with \cite[Theorem 2.1.7]{nesterov2018lectures} and \cite[Theorem 2.1.13]{nesterov2018lectures}, we obtain the corresponding lower bounds for the convex and strongly convex cases, respectively.

\section{Proof of Corollary~\ref{cor:ExtendT}}\label{Appe:ExtendT}
For simplicity, let $E_t^{(i)}$ denote the individual-token error, i.e., $E_t^{(i)}=\|\nabla f(x_t^{(i)})\|^2$ in the non-convex case and $E_t^{(i)}=f(x_t^{(i)})-f(x^\star)$ in convex and strongly convex cases.

We first consider the best-token criterion. Since $\min_{j\in[m]} E_t^{(j)}\le E_t^{(i)}$ for any $i\in[m]$, whenever token $i$ reaches accuracy $\epsilon$, the best-token criterion is also satisfied. Hence, $T_\epsilon^{\text{best}}\le T_\epsilon^{(i)}$ for any $i\in[m]$. Thus, the corresponding iteration complexity bounds described in Remark~\ref{remark:UPT} applies to $T_\epsilon^{\text{best}}$.

We next consider the mean-token criterion. Note that 
\begin{align*}
\frac{1}{m}\sum_{j\in[m]}f(x_t^{(j)}) - f(x^\star) =& \frac{1}{m}\sum_{j\in[m]}\left(f(x_t^{(j)}) - f(x^\star)\right)\\
=& \frac{1}{m}\sum_{j\in[m]} E_t^{(j)}.
\end{align*}
From \eqref{eq:nonconvexbesttoken}, for the non-convex case, the mean-token criterion is also given by $\frac{1}{m}\sum_{j\in[m]} E_t^{(j)}$.
Thus, in all three cases, the mean-token criterion is the average of the individual-token errors. Since the error bounds in \eqref{eq:nonconvexerror}, \eqref{eq:convexerror} and \eqref{eq:sconvexerror} hold for ever $j\in[m]$, the same bounds also hold for their average. Consequently, the corresponding iteration complexity bounds described in Remark~\ref{remark:UPT} also applies to $T_\epsilon^{\text{mean}}$.

\numberwithin{equation}{section}
\renewcommand{\theequation}{\Alph{section}.\arabic{equation}}

\end{document}